\documentclass[11pt]{article}
\usepackage[margin=1in]{geometry}
\usepackage{times}
\usepackage{natbib}
\usepackage[T1]{fontenc}
\usepackage[utf8]{inputenc}
\usepackage{microtype}
\usepackage{amsmath,amssymb,amsthm,mathtools,bm}
\usepackage{algorithm}
\usepackage{algpseudocode}
\usepackage{booktabs}
\usepackage{array}
\usepackage{graphicx}
\usepackage{xcolor}
\usepackage{enumitem}
\usepackage{tikz}
\usetikzlibrary{arrows.meta,positioning,shapes.geometric}
\usepackage[colorlinks=true,linkcolor=blue!60!black,citecolor=blue!60!black,
            urlcolor=blue!60!black]{hyperref}
\usepackage{cleveref}

\makeatletter
\providecommand{\theHALG@line}{}
\renewcommand{\theHALG@line}{\thealgorithm.\arabic{ALG@line}}
\makeatother

\newtheorem{proposition}{Proposition}
\newtheorem{theorem}{Theorem}
\newtheorem{lemma}{Lemma}
\newtheorem{corollary}{Corollary}
\theoremstyle{definition}
\newtheorem{definition}{Definition}
\newtheorem{example}{Example}
\theoremstyle{remark}
\newtheorem{remark}{Remark}

\newcommand{\R}{\mathbb{R}}
\newcommand{\E}{\mathbb{E}}

\newcommand{\tr}{\mathrm{tr}}
\newcommand{\diag}{\mathrm{diag}}
\newcommand{\rank}{\mathrm{rank}}
\newcommand{\T}{^{\!\top}}
\newcommand{\DPP}{\textsc{DPP}}
\newcommand{\OurMethod}{\textsc{CertDPP}}
\newcommand{\OPT}{\mathrm{OPT}}
\newcommand{\gap}{\Delta}

\definecolor{cBlue}{HTML}{2563eb}
\definecolor{cOrange}{HTML}{ea580c}
\definecolor{cGreen}{HTML}{059669}
\definecolor{cGray}{HTML}{6b7280}

\title{Spectral Certificates and Projection-DPP Rounding\\
for Determinantal MAP Selection}

\author{Richard Yi Da Xu\\
Hong Kong Baptist University \quad TadReamk Limited\\
\texttt{xuyida@hkbu.edu.hk} \quad \texttt{richard@tadreamk.com}}
\date{}

\begin{document}
\maketitle

\begin{abstract}
Selecting a fixed-size subset that maximizes the determinant of a positive
semidefinite kernel is the MAP problem for a size-constrained determinantal
point process and the classical maximum-entropy sampling problem.  It is
NP-hard.  A classical spectral bound states that no feasible subset can exceed
a ceiling given by the sum of the logarithms of the leading eigenvalues of the
kernel.  The same quantity is the exact optimum of the Stiefel relaxation of
the problem, so the continuous problem needs no new optimizer: it is solved by
the leading eigenspace.  We ask instead what this eigenspace says about the
discrete rounding problem.

The leading eigenvectors induce a projection DPP, a distribution that assigns
each subset of the target size a probability equal to the squared volume of
the corresponding rows of the eigenvector basis.  We show a pointwise
inequality: the amount by which any subset falls short of the spectral ceiling
is at most the negative log-probability of that subset under this projection
DPP.  Consequently the integrality gap is at most the min-entropy of the
distribution, with equality when the rank of the kernel equals the subset
size.  An exact projection-DPP sample has expected certified gap at most the
Shannon entropy of the distribution and, with any prescribed confidence,
attains a value within an additive term of the spectral ceiling that grows
only logarithmically with the number of candidate subsets and the inverse
failure probability.  This turns the spectral bound into a constructive
rounding guarantee and identifies the geometry that spectral decay alone
misses: localization of the leading subspace in the coordinate basis.

These results motivate \OurMethod: compute the leading eigenspace
matrix-free, draw one or more projection-DPP samples from the exact target law,
optionally improve the best sample by determinant-increasing swaps, and report
the gap between a verified spectral ceiling and the achieved log-determinant.
When only unverified Ritz values are available, we label the corresponding
quantity a numerical gap estimate rather than a rigorous certificate.
Controlled experiments verify the entropy
identities, compare the rounded solutions with exact MAP on small instances,
and confirm linear scaling in the ground-set size for the rounding stage.  The
certificate can be loose for diffuse eigenspaces; exposing that fact
quantitatively is part of the method rather than a hidden failure mode.
\end{abstract}

\section{Introduction}\label{sec:intro}

Determinantal point processes (\DPP s) provide a probability model for diverse
subset selection.  Given a positive semidefinite kernel $L\in\R^{n\times n}$,
an $L$-ensemble assigns a subset $S$ probability proportional to
$\det(L_S)$ \citep{kulesza2012dpp}.  If $L$ is a Gram matrix, this determinant
is the squared volume spanned by the selected feature vectors.  It is large
when the selected items are individually strong and collectively
nonredundant.

The deterministic size-constrained MAP problem is
\begin{equation}
  \OPT_k
  :=\max_{S\subseteq[n],\ |S|=k}\log\det(L_S).
  \label{eq:map}
\end{equation}
The same optimization problem is known as maximum-entropy sampling in
statistics and experimental design.  It is NP-hard
\citep{ko1995exact,civril2009selecting}, and strong inapproximability results
are known for related determinant-maximization formulations
\citep{ohsaka2021inapproximability}.  Greedy, exchange, continuous-relaxation,
and branch-and-bound methods cover useful regimes, but exact solution remains
combinatorial \citep{chen2018fastdpp,han2017faster,anstreicher2020efficient}.

There is a simple and old upper bound.  If
$\lambda_1(L)\ge\cdots\ge\lambda_n(L)\ge0$, then eigenvalue interlacing gives
\begin{equation}
  \log\det(L_S)\le
  U_k:=\sum_{i=1}^k\log\lambda_i(L)
  \qquad(|S|=k).
  \label{eq:spectral-certificate}
\end{equation}
This spectral bound appeared in early work on maximum-entropy sampling
\citep{ko1995exact} and has since been strengthened by partition, masking, and
convex-optimization bounds \citep{anstreicher2004masked,burer2007factored,
anstreicher2020efficient}.  It is also exactly the optimum of the continuous
Stiefel relaxation
\begin{equation}
  \max_{V\in\R^{n\times k},\ V\T V=I_k}
       \log\det(V\T L V).
  \label{eq:stiefel}
\end{equation}
Thus \eqref{eq:stiefel} is not a difficult nonconvex program, and an iterative
nonlinear solver is unnecessary.  Its optimizer is any basis of the leading
$k$-dimensional eigenspace.

This observation changes the research question.  The useful object is not a
new solver for \eqref{eq:stiefel}, but the relationship between its dense
eigenspace and an actual coordinate subset.  Two questions become central:
\begin{enumerate}[leftmargin=1.7em,itemsep=2pt]
  \item When is the spectral certificate $U_k$ close to the discrete optimum?
  \item Can the leading eigenspace be rounded by a scalable procedure that
        inherits a guarantee from the same certificate?
\end{enumerate}

Our answer uses the projection DPP associated with the leading eigenspace.
Let $U\in\R^{n\times k}$ have orthonormal columns spanning that eigenspace.
The quantities
\[
  p_U(S):=\det(U_S)^2,\qquad |S|=k,
\]
sum to one by Cauchy--Binet and therefore form a probability distribution.
They are precisely the probabilities of the rank-$k$ projection DPP with
kernel $UU\T$.  The determinant $\det(U_S)^2$ measures how much coordinate
volume the subset $S$ captures from the leading subspace.

The central inequality of this paper is
\begin{equation}
  \log\det(L_S)\ge U_k+\log p_U(S).
  \label{eq:central-preview}
\end{equation}
It has several consequences.  First, the best coordinate-volume probability
controls the integrality gap; in the rank-$k$ case it characterizes that gap
exactly.  Second, sampling from $p_U$ makes the classical counting floor
$U_k-\log\binom nk$ algorithmic rather than merely existential.  Third, the
Shannon entropy of $p_U$ bounds the expected gap of the rounded set.  These
statements explain why eigenvalue decay is not enough: the missing variable is
how the leading eigenspace sits relative to the coordinate axes.

\Cref{fig:pipeline} gives the entire argument at a glance.  The top
eigenpairs play two distinct roles.  Their eigenvalues give a scalar upper
bound $U_k$, while their eigenvectors give a distribution over coordinate
subsets.  The pointwise inequality connects these two branches.  It converts
the probability of a sampled subset into a lower bound on its determinant;
evaluating that determinant on the full kernel then produces an a posteriori
optimality certificate.

\begin{figure}[t]
\centering
\resizebox{\linewidth}{!}{%
\begin{tikzpicture}[
  node distance=9mm and 12mm,
  >=Latex,
  every node/.style={font=\small,align=center},
  input/.style={draw=cGray,very thick,rounded corners=2pt,fill=black!3,
                minimum width=28mm,minimum height=12mm},
  spectral/.style={draw=cBlue,very thick,rounded corners=2pt,fill=cBlue!7,
                   minimum width=34mm,minimum height=13mm},
  random/.style={draw=cOrange,very thick,rounded corners=2pt,fill=cOrange!8,
                 minimum width=37mm,minimum height=13mm},
  output/.style={draw=cGreen,very thick,rounded corners=2pt,fill=cGreen!8,
                 minimum width=38mm,minimum height=13mm},
  flow/.style={->,thick}
]
\node[input] (input) {PSD kernel $L$\\target size $k$};
\node[spectral,right=of input] (eig) {leading eigenpairs\\$U,\ \Lambda_{1:k}$};
\node[spectral,above right=7mm and 13mm of eig] (upper)
  {spectral ceiling\\$U_k=\log\det\Lambda_{1:k}$};
\node[random,below right=7mm and 13mm of eig] (dpp)
  {projection DPP\\$p_U(S)=\det(U_S)^2$};
\node[random,right=of dpp] (sample) {exact sample\\$S\sim p_U$};
\node[output,right=of sample] (value) {full-kernel evaluation\\$f(S)=\log\det(L_S)$};
\node[output,above=of value] (cert)
  {a posteriori certificate\\$0\le\OPT_k-f(S)\le U_k-f(S)$};

\draw[flow] (input) -- node[above]{eigensolve} (eig);
\draw[flow,cBlue] (eig) -- (upper);
\draw[flow,cOrange] (eig) -- (dpp);
\draw[flow,cOrange] (dpp) -- node[above]{round} (sample);
\draw[flow,cGreen] (sample) -- (value);
\draw[flow,cGreen] (value) -- (cert);
\draw[flow,cBlue] (upper) -- (cert);
\draw[->,densely dashed,thick,cOrange]
  (dpp.south east) to[bend right=24]
  node[below,sloped,text=cOrange!80!black]
  {$f(S)\ge U_k+\log p_U(S)$} (value.south west);
\end{tikzpicture}%
}
\caption{\textbf{How the certificate-and-rounding argument fits together.}
The eigenvalues and eigenvectors of the same leading eigenspace generate two
branches.  The blue branch is an upper bound on the unknown discrete optimum;
the orange branch is a distribution used to construct a subset.  The dashed
arrow is the central pointwise inequality.  The green branch evaluates the
returned subset on the original kernel and compares it with the blue upper
bound.}
\label{fig:pipeline}
\end{figure}
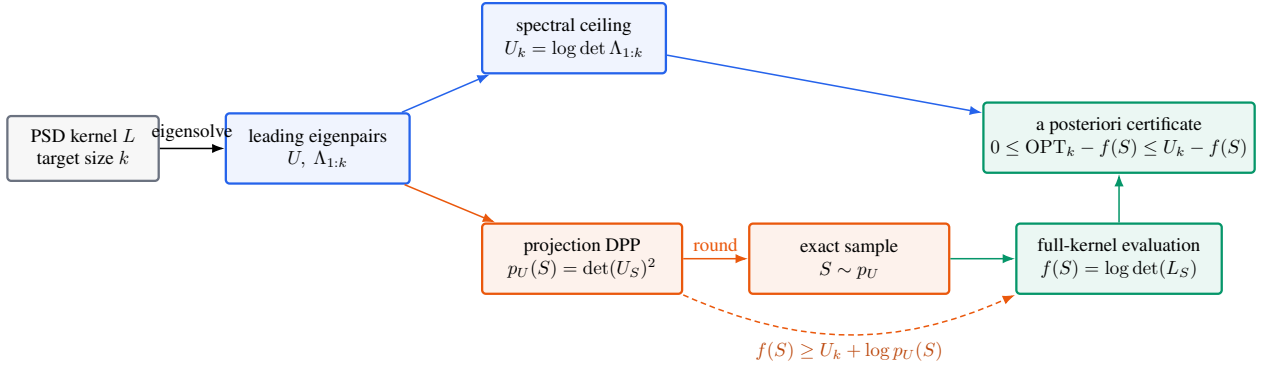

\paragraph{Contributions.}
The paper makes four focused contributions.
\begin{enumerate}[leftmargin=*]
  \item \textbf{A coordinate-volume interpretation of the spectral bound.}
    We relate every discrete determinant to the projection-DPP probability of
    the same subset through \eqref{eq:central-preview}.  This separates the
    spectral mass $\prod_{i\le k}\lambda_i(L)$ from the coordinate-rounding
    volume $\det(U_S)^2$.
  \item \textbf{Entropy characterizations and guarantees.}
    The spectral integrality gap is bounded by the min-entropy
    $H_\infty(p_U)$ and equals it for rank-$k$ kernels.  Projection-DPP
    rounding has expected gap at most $H(p_U)$ and a universal high-probability
    additive guarantee.
  \item \textbf{A certified matrix-free pipeline.}
    The pipeline combines a leading-eigenspace computation, exact
    projection-DPP rounding, optional monotone local refinement, and the
    candidate-specific certificate $U_k-\log\det(L_S)$.  The guarantee belongs
    to the initial sample and therefore survives any improving refinement.
  \item \textbf{Experiments designed around certificate tightness.}
    We verify the entropy identities in a controlled family, compare with
    exact MAP rather than only with other relaxations, apply equal local-search
    budgets to all initializers, and separately measure the scaling of exact
    projection-DPP rounding.
\end{enumerate}

\paragraph{Scope.}
This is a theory-and-algorithm paper about certificates and rounding.  It does
not claim that the Stiefel relaxation or the spectral bound is new, and it does
not introduce a solver for the fixed-kernel relaxation.  The experiments are
controlled numerical validations, not a production-scale downstream data
curation study.  This narrower scope lets us state precisely what is guaranteed
and where the remaining looseness lies.
This report is intended as a substantial revision and replacement of
arXiv:2606.19411v1: it removes the earlier emphasis on an unnecessary nonlinear
solver for the fixed-kernel relaxation and replaces heuristic rounding claims
with the projection-DPP guarantees developed here.

\section{Reader's Guide and Notation}\label{sec:notation}

This section fixes the notation used throughout the report and explains the
three levels of the argument.  Readers interested mainly in implementation
can read \cref{sec:problem,sec:projection,sec:algorithm} in that order.  Readers
interested in the guarantees should additionally read \cref{sec:theory}; the
fully worked example in \cref{ex:three-by-two} can be used to check every
formula numerically.

\subsection*{Sets, selectors, and principal submatrices}
The ground set is $[n]:=\{1,\ldots,n\}$.  A candidate subset
$S=\{i_1,\ldots,i_k\}$ always has exactly $k$ distinct indices.  We place the
indices in increasing order when writing matrices.  The coordinate selector
is
\begin{equation}
  E_S=[e_{i_1},\ldots,e_{i_k}]\in\R^{n\times k},
  \qquad E_S\T E_S=I_k.
  \label{eq:selector-definition}
\end{equation}
For an $n\times n$ matrix $L$, its principal submatrix on $S$ is
$L_S:=E_S\T L E_S\in\R^{k\times k}$.  For an $n\times k$ matrix $U$, the
notation $U_S:=E_S\T U\in\R^{k\times k}$ means the rows indexed by $S$.
These two uses of the subscript are compatible:
\begin{equation}
  (UU\T)_S=U_SU_S\T.
  \label{eq:submatrix-compatibility}
\end{equation}

\subsection*{Spectral notation and logarithms}
The kernel $L$ is real symmetric and positive semidefinite (PSD), written
$L\succeq0$.  Its eigenvalues are always ordered as
$\lambda_1(L)\ge\cdots\ge\lambda_n(L)\ge0$.  We assume
$\lambda_k(L)>0$ so that the target rank is meaningful.  The leading
eigenvector matrix $U=[u_1,\ldots,u_k]$ has orthonormal columns,
$U\T U=I_k$, and
$\Lambda_{1:k}:=\diag(\lambda_1,\ldots,\lambda_k)$.  All logarithms are
natural logarithms and are therefore measured in nats.  We adopt
$\log0=-\infty$, so a singular selected submatrix has objective value
$-\infty$.

If $k<n$ and $\lambda_k(L)=\lambda_{k+1}(L)$, the maximizing
$k$-dimensional invariant
subspace is not unique.  In that case we fix one such subspace, including a
deterministic eigensolver tie-breaking rule, and all quantities
$p_U,H(p_U)$, and $H_\infty(p_U)$ are conditional on that choice.  They are
basis-invariant within the fixed subspace but need not be invariant across
different choices inside the tied eigenspace.  The scalar ceiling $U_k$ and all
pointwise theorems remain valid for every fixed maximizing choice.  When an
intrinsic leading subspace is desired, one may instead assume the eigengap
$\lambda_k(L)>\lambda_{k+1}(L)$.

\subsection*{The three quantities that should not be confused}
For a returned subset $S$, the analysis uses three different gaps:
\begin{center}
\small
\begin{tabular}{p{.19\linewidth}p{.28\linewidth}p{.43\linewidth}}
\toprule
Quantity & Formula & Interpretation \\
\midrule
Unknown MAP optimum & $\OPT_k=\max_{|T|=k}\log\det(L_T)$
  & Best discrete value; expensive to compute in general.\\
Spectral ceiling & $U_k=\sum_{i=1}^k\log\lambda_i(L)$
  & Efficient upper bound on $\OPT_k$ and exact Stiefel-relaxation value.\\
Reported certificate & $\gap(S)=U_k-\log\det(L_S)$
  & Computable upper bound on the unknown suboptimality
    $\OPT_k-\log\det(L_S)$.\\
\bottomrule
\end{tabular}
\end{center}

\paragraph{Why $\OPT_k\le U_k$.}
Fix any size-$k$ subset $T$, and let
$\mu_1(L_T)\ge\cdots\ge\mu_k(L_T)\ge0$ be the eigenvalues of its principal
submatrix.  Cauchy eigenvalue interlacing gives
\begin{equation}
  \mu_i(L_T)\le\lambda_i(L),
  \qquad i=1,\ldots,k.
  \label{eq:reader-interlacing}
\end{equation}
If $L_T$ is nonsingular, taking logarithms and summing yields
\begin{equation}
  \log\det(L_T)
  =\sum_{i=1}^k\log\mu_i(L_T)
  \le\sum_{i=1}^k\log\lambda_i(L)
  =U_k.
  \label{eq:reader-subset-upper-bound}
\end{equation}
If $L_T$ is singular, the same inequality holds under the convention
$\log\det(L_T)=-\infty$.  Because the inequality holds for every feasible
$T$, taking the maximum over all size-$k$ subsets proves
\begin{equation}
  \boxed{\quad
  \OPT_k=\max_{|T|=k}\log\det(L_T)\le U_k.
  \quad}
  \label{eq:reader-opt-upper-bound}
\end{equation}
Consequently, every returned subset $S$ satisfies the complete ordering
\begin{equation}
  \log\det(L_S)\le\OPT_k\le U_k,
  \label{eq:reader-complete-ordering}
\end{equation}
which is precisely why
$\gap(S)=U_k-\log\det(L_S)$ is an upper bound on the unknown suboptimality
$\OPT_k-\log\det(L_S)$.

Thus a certificate of $0.1$ nats proves that the returned determinant is at
least $e^{-0.1}\approx0.905$ times the optimal determinant.  More generally,
$\gap(S)\le\eta$ implies
\begin{equation}
  \det(L_S)\ge e^{-\eta}\max_{|T|=k}\det(L_T).
  \label{eq:additive-to-multiplicative}
\end{equation}
The certificate is additive in log-determinant units and multiplicative in
determinant units.

\section{Related Work}\label{sec:related}

\paragraph{DPP-MAP and MESP.}
The size-$k$ DPP-MAP problem and maximum-entropy sampling (MESP) share
the same determinant objective.  \citet{ko1995exact} established NP-hardness,
an eigenvalue-interlacing upper bound, and an exact branch-and-bound method.
Subsequent work developed spectral partitions, masked spectral bounds, and
strong convex upper bounds \citep{anstreicher2004masked,burer2007factored,
anstreicher2020efficient}.  This literature is essential context: the quantity
$U_k$ in \eqref{eq:spectral-certificate} is classical.  Our focus is different
from tightening $U_k$ through a stronger optimization relaxation; we use the
leading eigenspace behind $U_k$ to construct and analyze a discrete rounding
distribution.

In machine learning, \DPP s were systematized by \citet{kulesza2012dpp}.
Greedy and accelerated greedy MAP methods operate directly on marginal
determinant gains \citep{chen2018fastdpp,han2017faster}.  The softmax extension
of \citet{gillenwater2012near} relaxes a binary membership vector to a capped
simplex and is exact at the binary vertices.  Design relaxations similarly
optimize weighted log determinants followed by randomized or exchange rounding
\citep{nikolov2015randomized,singh2020doptimal}.  The Stiefel and simplex
relaxations enlarge the discrete feasible set in different ways; neither
feasible-set geometry alone enforces semantic diversity.  In both cases the
kernel and determinant objective encode diversity.

\paragraph{Volume sampling and projection DPPs.}
Volume sampling selects rows with probability proportional to their squared
spanned volume \citep{deshpande2006matrix,deshpande2010efficient}.  When
$U\T U=I_k$ and exactly $k$ rows are selected, volume sampling is the
projection DPP with kernel $UU\T$.  Projection DPPs have fixed cardinality and
their row leverage scores are the marginal inclusion probabilities
\citep{kulesza2012dpp}.  Efficient exact and approximate samplers have been
developed using sequential orthogonalization, intermediate sampling, MCMC, and
rejection sampling \citep{gautier2017zonotope,derezinski2019fast,
barthelme2023faster}.  \citet{belhadji2020column} use a projection DPP for
column subset selection and analyze reconstruction error.  We instead analyze
the DPP-MAP log-determinant value obtained when the projection kernel is the
leading eigenspace of the target kernel $L$.

\paragraph{Maximum-volume submatrices and rank-revealing selection.}
Selecting a maximum-volume row or column submatrix is NP-hard
\citep{civril2009selecting}.  Pivoted and strong rank-revealing QR methods give
deterministic practical alternatives \citep{gu1996efficient}, while randomized
column-subset methods use leverage or volume information.  Our min-entropy
quantity is a probabilistic representation of the same geometry: maximizing
$p_U(S)=\det(U_S)^2$ is exactly the maximum-volume row-submatrix problem for
the leading eigenbasis.

\section{Problem, Relaxation, and Certificate}\label{sec:problem}

Let $L\succeq0$ have rank at least $k$, so at least one size-$k$ principal
minor is positive.  We use the convention $\log\det(L_S)=-\infty$ when
$L_S$ is singular.  Let
\[
  L=U_{1:k}\Lambda_{1:k}U_{1:k}\T+U_{>k}\Lambda_{>k}U_{>k}\T
   =:L_k+R_k,
\]
where $\Lambda_{1:k}=\diag(\lambda_1,\dots,\lambda_k)$ and
$U:=U_{1:k}\in\R^{n\times k}$.  Then $R_k\succeq0$.

\subsection{Why the determinant is a volume objective}
Because $L\succeq0$, there is a feature matrix $B\in\R^{n\times d}$ such
that $L=BB\T$; row $b_i\T$ represents item $i$.  For a subset $S$,
\begin{equation}
  L_S=B_SB_S\T,
  \qquad
  \det(L_S)=\det(B_SB_S\T).
  \label{eq:gram-volume}
\end{equation}
When the selected rows are linearly independent, this determinant is the
squared $k$-dimensional volume of the parallelepiped spanned by them.  If two
selected items point in the same feature direction, the volume shrinks; if
the selected rows are linearly dependent, it is zero.  Thus the objective
combines item magnitude with nonredundancy.  The logarithm is used because it
turns products of eigenvalues into sums and makes determinant ratios into
additive gaps.

\subsection{Embedding a subset in the Stiefel manifold}

For a subset $S=\{i_1,\dots,i_k\}$, let $E_S$ denote the $n\times k$
coordinate selector with columns $e_{i_1},\dots,e_{i_k}$.  It satisfies
$E_S\T E_S=I_k$ and $E_S\T L E_S=L_S$.  Consequently \eqref{eq:stiefel}
is a relaxation of \eqref{eq:map}: every subset produces a feasible Stiefel
matrix, but a general Stiefel matrix may be a dense linear combination of all
$n$ coordinates.  This distinction is the source of the integrality gap.

\begin{proposition}[Classical spectral certificate]\label{prop:spectral}
If $\lambda_k(L)>0$, then
\begin{equation}
  \max_{V\T V=I_k}\log\det(V\T L V)
  =U_k:=\sum_{i=1}^k\log\lambda_i(L),
  \label{eq:stiefel-solution}
\end{equation}
and the maximum is attained by any orthonormal basis of the leading
$k$-dimensional eigenspace.  Hence $\OPT_k\le U_k$.
\end{proposition}

\begin{proof}
Fix a feasible $V$ and let
$\mu_1\ge\cdots\ge\mu_k\ge0$ be the eigenvalues of the compressed matrix
$V\T L V$.  Poincar\'e separation gives
\begin{equation}
  \mu_i\le\lambda_i(L),\qquad i=1,\ldots,k.
  \label{eq:poincare-coordinatewise}
\end{equation}
Multiplying these $k$ inequalities yields
\begin{equation}
  \det(V\T L V)=\prod_{i=1}^k\mu_i
  \le\prod_{i=1}^k\lambda_i(L).
  \label{eq:compressed-determinant-bound}
\end{equation}
For $V=UQ$, where $Q\in\R^{k\times k}$ is orthogonal,
\begin{align*}
  V\T L V
  &=Q\T U\T L UQ
    =Q\T\Lambda_{1:k}Q,\\
  \det(V\T L V)
  &=\det(\Lambda_{1:k})
    =\prod_{i=1}^k\lambda_i(L),
\end{align*}
so the upper bound is attained.  Taking logarithms gives
\eqref{eq:stiefel-solution}.  Finally, every coordinate selector $E_S$ is
feasible and satisfies $E_S\T L E_S=L_S$, so maximizing over all Stiefel
matrices can only exceed the discrete maximum.
\end{proof}

The proposition explains why no iterative nonlinear fixed-point method is
needed for the fixed kernel studied here.  One needs only the top $k$
eigenpairs, which can be computed without forming a dense $n\times n$ matrix
when products $x\mapsto Lx$ are available.

\begin{definition}[A posteriori spectral certificate]\label{def:certificate}
For any returned size-$k$ subset $S$, define
\begin{equation}
  \gap(S):=U_k-\log\det(L_S).
  \label{eq:posterior-gap}
\end{equation}
Then
\begin{equation}
  0\le \OPT_k-\log\det(L_S)\le\gap(S).
  \label{eq:posterior-certificate}
\end{equation}
\end{definition}

The certificate is algorithm-independent: it applies to greedy, exchange,
simplex rounding, projection-DPP rounding, or any externally supplied subset.
Its usefulness is empirical and instance-dependent.  A small $\gap(S)$ proves
near-optimality; a large value honestly reports that this relaxation cannot
certify the candidate.

To see the direction of the inequalities, write
\begin{equation}
  \underbrace{\log\det(L_S)}_{\text{returned value}}
  \ \le\ 
  \underbrace{\OPT_k}_{\text{unknown discrete optimum}}
  \ \le\ 
  \underbrace{U_k}_{\text{computable spectral ceiling}}.
  \label{eq:certificate-sandwich}
\end{equation}
Subtracting the leftmost term from all three quantities gives
\eqref{eq:posterior-certificate}.  Importantly, $\gap(S)$ is not an estimate
of the error; it is a valid worst-case upper bound on the error whenever the
eigenvalue ceiling is computed rigorously.

\section{The Leading-Eigenspace Projection DPP}\label{sec:projection}

The leading eigenvectors contain more information than the scalar $U_k$.
Their $k\times k$ row minors describe how well the leading subspace can be
represented by $k$ coordinates.

\begin{definition}[Projection-volume distribution]\label{def:projection}
For $U\in\R^{n\times k}$ with $U\T U=I_k$, define
\begin{equation}
  p_U(S):=\det(U_S)^2=\det\bigl((UU\T)_S\bigr),
  \qquad |S|=k.
  \label{eq:projection-probability}
\end{equation}
\end{definition}

The determinant in \eqref{eq:projection-probability} is well-defined because
$U_S$ is square: the projection DPP has rank $k$ and we consider subsets of
exactly $k$ rows.  Expanding the second expression gives
\begin{equation}
  \det\bigl((UU\T)_S\bigr)
  =\det(U_SU_S\T)
  =\det(U_S)\det(U_S\T)
  =\det(U_S)^2.
  \label{eq:projection-minor-equality}
\end{equation}
Thus $p_U(S)$ is nonnegative.  It is zero precisely when the selected rows of
$U$ are linearly dependent.

By Cauchy--Binet,
\begin{equation}
  \sum_{|S|=k}p_U(S)
  =\sum_{|S|=k}\det(U_S)^2
  =\det(U\T U)=1.
  \label{eq:cauchy-binet-prob}
\end{equation}
Thus $p_U$ is a probability distribution: the projection DPP with marginal
kernel $K=UU\T$.  A projection-DPP sample has size $k$ almost surely.  The
marginal probability of including item $i$ is
$K_{ii}=\|e_i\T U\|_2^2$, the usual row leverage score.  Unlike independent
leverage sampling, the joint distribution contains determinant-induced
negative dependence and assigns zero probability to linearly dependent row
sets.

The marginal identity should be read carefully: $K_{ii}$ is the probability
that item $i$ appears, not the probability of a whole subset.  In general,
sampling $k$ distinct indices independently from weights proportional to
$K_{ii}$ does \emph{not} produce $p_U$.  The determinant couples the choices:
once selected rows already span a direction, another row in the same
direction has little or no conditional probability.

The distribution does not depend on the chosen basis: replacing $U$ by $UQ$
with $Q\T Q=I$ leaves $p_U(S)$ unchanged.  It is therefore an invariant of the
leading subspace.  Indeed,
\begin{equation}
  \det((UQ)_S)^2
  =\det(U_SQ)^2
  =\det(U_S)^2\det(Q)^2
  =\det(U_S)^2.
  \label{eq:basis-invariance}
\end{equation}

\begin{definition}[Roundability entropies]\label{def:entropies}
Let $\mathcal S_k=\{S\subseteq[n]:|S|=k\}$.  Define
\begin{align}
  H(p_U)&:=-\sum_{S\in\mathcal S_k}p_U(S)\log p_U(S),
  \label{eq:shannon}\\
  H_\infty(p_U)&:=-\log\max_{S\in\mathcal S_k}p_U(S).
  \label{eq:minentropy}
\end{align}
We use $0\log0=0$.  Since $p_U$ is supported on at most $\binom nk$ sets,
\begin{equation}
  0\le H_\infty(p_U)\le H(p_U)\le\log\binom nk.
  \label{eq:entropy-chain}
\end{equation}
\end{definition}

Low entropy means that a small number of coordinate subsets capture most of
the leading-subspace volume.  High entropy means that volume is diffuse over
many subsets.  This is the relevant geometry for rounding.

The two entropies answer different operational questions.  If
$p_{\max}:=\max_Sp_U(S)$, then
$H_\infty(p_U)=-\log p_{\max}$ measures the penalty of the single best
coordinate representation.  The Shannon entropy
$H(p_U)=\E_{S\sim p_U}[-\log p_U(S)]$ measures the average penalty paid by one
random draw.  Since an average cannot be smaller than the minimum value of
$-\log p_U(S)$, $H_\infty\le H$; since no distribution on
$\binom nk$ outcomes has entropy exceeding the logarithm of its support size,
$H\le\log\binom nk$.

\section{Main Results: Gap Geometry and Constructive Rounding}
\label{sec:theory}

\begin{theorem}[Spectral mass times coordinate volume]\label{thm:pointwise}
Let $L\succeq0$ with $\lambda_k(L)>0$, let $U$ contain its leading $k$
eigenvectors, and let $U_k$ be defined by
\eqref{eq:spectral-certificate}.  For every $S\in\mathcal S_k$,
\begin{equation}
  \boxed{\quad
  \det(L_S)\ge
  \left(\prod_{i=1}^k\lambda_i(L)\right)p_U(S)
  \quad}
  \label{eq:pointwise-det}
\end{equation}
or, equivalently whenever $p_U(S)>0$,
\begin{equation}
  \log\det(L_S)\ge U_k+\log p_U(S).
  \label{eq:pointwise-log}
\end{equation}
If $\rank(L)=k$, equality holds for every $S$.
\end{theorem}

\begin{proof}
We separate the argument into explicit steps.
\begin{enumerate}[leftmargin=2.1em,itemsep=5pt]
  \item \emph{Split off the leading spectral component.}
  Write
  \begin{equation}
    L=L_k+R_k,
    \qquad
    L_k=U\Lambda_{1:k}U\T,
    \qquad R_k\succeq0.
    \label{eq:proof-spectral-split}
  \end{equation}

  \item \emph{Restrict the split to the selected coordinates.}
  Premultiplying and postmultiplying by $E_S\T$ and $E_S$ gives
  \begin{equation}
    L_S
    =U_S\Lambda_{1:k}U_S\T+(R_k)_S,
    \qquad (R_k)_S\succeq0.
    \label{eq:proof-principal-split}
  \end{equation}
  Hence $L_S\succeq U_S\Lambda_{1:k}U_S\T$.

  \item \emph{Use determinant monotonicity.}
  Suppose first that $p_U(S)>0$.  Then $U_S$ is nonsingular and
  $A:=U_S\Lambda_{1:k}U_S\T\succ0$.  With
  $C:=(R_k)_S\succeq0$,
  \begin{align}
    \det(L_S)
    &=\det(A+C)\notag\\
    &=\det(A)\det\!\left(I_k+A^{-1/2}CA^{-1/2}\right)
      \ge\det(A),
    \label{eq:proof-monotonicity}
  \end{align}
  because every eigenvalue of the second factor is at least one.

  \item \emph{Factor the leading-subspace determinant.}
  Using multiplicativity and $\det(U_S\T)=\det(U_S)$,
  \begin{align}
    \det(A)
    &=\det(U_S)\det(\Lambda_{1:k})\det(U_S\T)\notag\\
    &=\det(U_S)^2\prod_{i=1}^k\lambda_i(L)
      =p_U(S)\prod_{i=1}^k\lambda_i(L).
    \label{eq:proof-factorization}
  \end{align}
\end{enumerate}
Combining the last two steps proves \eqref{eq:pointwise-det}.  Taking logs
gives \eqref{eq:pointwise-log}.  If $p_U(S)=0$, the right-hand side of the
determinant inequality is zero and the inequality follows from $L_S\succeq0$;
the log form is then interpreted as the vacuous lower bound $-\infty$.  When
$\rank(L)=k$, the tail $R_k$ vanishes, so
$L_S=U_S\Lambda_{1:k}U_S\T$ and equality holds for every $S$.
\end{proof}

The theorem cleanly separates two effects.  The eigenvalues determine the
best unconstrained $k$-dimensional volume, while $p_U(S)$ measures the price of
representing that subspace using the coordinate subset $S$.  The neglected
spectral tail $R_k$ can only improve the discrete determinant above this
leading-subspace floor.

\subsection{A fully worked three-item, two-selection example}

\begin{example}[Uniform coordinate volumes in a rank-two kernel]
\label{ex:three-by-two}
Let
\begin{equation}
  U=
  \begin{bmatrix}
    1/\sqrt2 & 1/\sqrt6\\
   -1/\sqrt2 & 1/\sqrt6\\
    0        & -2/\sqrt6
  \end{bmatrix},
  \qquad
  \Lambda=\begin{bmatrix}9&0\\0&4\end{bmatrix},
  \qquad
  L=U\Lambda U\T.
  \label{eq:worked-example-definition}
\end{equation}
The two columns of $U$ have unit norm and inner product zero, so
$U\T U=I_2$.  Direct multiplication gives
\begin{equation}
  L=
  \begin{bmatrix}
    31/6 & -23/6 & -4/3\\
    -23/6 & 31/6 & -4/3\\
    -4/3 & -4/3 & 8/3
  \end{bmatrix}.
  \label{eq:worked-example-kernel}
\end{equation}
There are only $\binom32=3$ candidate subsets.  Every quantity can therefore
be enumerated:
\begin{center}
\small
\begin{tabular}{ccccc}
\toprule
$S$ & $U_S$ & $\det(U_S)$ & $p_U(S)$ & $\det(L_S)$\\
\midrule
$\{1,2\}$ &
$\begin{bmatrix}1/\sqrt2&1/\sqrt6\\-1/\sqrt2&1/\sqrt6\end{bmatrix}$
& $1/\sqrt3$ & $1/3$ & $12$\\[8pt]
$\{1,3\}$ &
$\begin{bmatrix}1/\sqrt2&1/\sqrt6\\0&-2/\sqrt6\end{bmatrix}$
& $-1/\sqrt3$ & $1/3$ & $12$\\[8pt]
$\{2,3\}$ &
$\begin{bmatrix}-1/\sqrt2&1/\sqrt6\\0&-2/\sqrt6\end{bmatrix}$
& $1/\sqrt3$ & $1/3$ & $12$\\
\bottomrule
\end{tabular}
\end{center}
The probabilities sum to one, exactly as Cauchy--Binet predicts.  The
spectral ceiling is
\begin{equation}
  U_2=\log9+\log4=\log36.
  \label{eq:worked-example-ceiling}
\end{equation}
For every subset, the pointwise theorem reads
\begin{equation}
  \det(L_S)=12
  =36\left(\frac13\right)
  =\det(\Lambda)p_U(S),
  \label{eq:worked-example-pointwise}
\end{equation}
or, after taking logarithms,
\begin{equation}
  \log12=\log36+\log(1/3).
  \label{eq:worked-example-log}
\end{equation}
Because the kernel has rank $k=2$, equality is expected.  The discrete
optimum is $\OPT_2=\log12$, while
\begin{equation}
  U_2-\OPT_2=\log36-\log12=\log3.
  \label{eq:worked-example-gap}
\end{equation}
The distribution is uniform on three sets, so
$H_\infty(p_U)=H(p_U)=\log3$.  Thus the integrality gap, min-entropy, and
expected projection-DPP certificate are all the same number in this example.
One exact projection-DPP draw returns each pair with probability $1/3$, and
every returned pair has reported certificate $\gap(S)=\log3$.
\end{example}

This example also clarifies the word ``projection.''  The two-dimensional
leading eigenspace is a plane in $\R^3$.  No pair of coordinate axes is
preferred: all three $2\times2$ row minors have the same squared area.  The
continuous optimizer captures area $36$, while requiring two coordinate rows
costs a factor of $1/3$.

\subsection{From the pointwise inequality to an integrality-gap bound}
The \emph{integrality gap} of the spectral relaxation is the difference
between the continuous ceiling and the best discrete value:
\begin{equation}
  g_k(L):=U_k-\OPT_k\ge0.
  \label{eq:integrality-gap-definition}
\end{equation}
It is a property of the instance $(L,k)$, whereas $\gap(S)$ is a property of
a particular returned subset.  Since $\OPT_k\ge\log\det(L_S)$ for every
$S$, one always has
\begin{equation}
  0\le g_k(L)\le\gap(S).
  \label{eq:instance-vs-candidate-gap}
\end{equation}
To obtain the strongest consequence of the pointwise theorem, choose the
subset with the largest projection probability.  This is exactly what turns
$p_{\max}$ into min-entropy.

\begin{corollary}[Min-entropy controls the integrality gap]
\label{cor:minentropy-gap}
Under the assumptions of \cref{thm:pointwise},
\begin{equation}
  0\le U_k-\OPT_k\le H_\infty(p_U)\le\log\binom nk.
  \label{eq:minentropy-gap}
\end{equation}
If $\rank(L)=k$, then
\begin{equation}
  U_k-\OPT_k=H_\infty(p_U).
  \label{eq:rankk-exact-gap}
\end{equation}
\end{corollary}

\begin{proof}
Let $S_\infty\in\arg\max_S p_U(S)$.  Apply the pointwise theorem to this
specific set:
\begin{align}
  \OPT_k
  &\ge\log\det(L_{S_\infty})
    &&\text{because $\OPT_k$ is the maximum over all sets},\notag\\
  &\ge U_k+\log p_U(S_\infty)
    &&\text{by \cref{thm:pointwise}},\notag\\
  &=U_k+\log\max_Sp_U(S)
    =U_k-H_\infty(p_U).
  \label{eq:minentropy-proof-chain}
\end{align}
Rearranging gives $U_k-\OPT_k\le H_\infty(p_U)$.  Combine this with
$\OPT_k\le U_k$ and \eqref{eq:entropy-chain}.  If $L$ has rank $k$, the
pointwise relation is an equality, so maximizing $\det(L_S)$ is exactly the
same as maximizing $p_U(S)$, and the gap is exactly
$-\log\max_Sp_U(S)$.
\end{proof}

The universal counting floor
$\OPT_k\ge U_k-\log\binom nk$ follows immediately, but
\cref{cor:minentropy-gap} is sharper conceptually and often numerically: the
width is determined by the actual leading eigenspace, not only by $n$ and
$k$.

Exponentiating the corollary gives a determinant-scale statement that is
sometimes easier to interpret:
\begin{equation}
  \max_{|S|=k}\det(L_S)
  \ge e^{-H_\infty(p_U)}\prod_{i=1}^k\lambda_i(L)
  =\left(\max_{|S|=k}p_U(S)\right)
    \prod_{i=1}^k\lambda_i(L).
  \label{eq:minentropy-multiplicative}
\end{equation}

\begin{theorem}[Projection-DPP rounding guarantees]\label{thm:rounding}
Draw $S\sim p_U$, the projection DPP of the leading eigenspace.  Then:
\begin{enumerate}[leftmargin=*,itemsep=4pt]
  \item \textbf{Expected certificate:}
  \begin{equation}
    \E\!\left[U_k-\log\det(L_S)\right]
    \le H(p_U)\le\log\binom nk.
    \label{eq:expected-gap}
  \end{equation}
  \item \textbf{High-probability value:} for every $\delta\in(0,1)$,
  \begin{equation}
    \Pr\!\left[
      \log\det(L_S)
      \ge U_k-\log\!\left(\frac{\binom nk}{\delta}\right)
    \right]\ge1-\delta.
    \label{eq:high-prob}
  \end{equation}
\end{enumerate}
If a post-processing method returns $S'$ with
$\log\det(L_{S'})\ge\log\det(L_S)$, both guarantees remain valid with $S'$ in
place of $S$.
\end{theorem}

\begin{proof}
For every set in the support of $p_U$, rearranging the pointwise inequality
gives
\begin{equation}
  U_k-\log\det(L_S)\le-\log p_U(S).
  \label{eq:samplewise-surprisal-bound}
\end{equation}
The quantity on the right is the self-information, or surprisal, of the
sampled set.  Taking expectations gives
\begin{align}
  \E_{S\sim p_U}[U_k-\log\det(L_S)]
  &\le\E_{S\sim p_U}[-\log p_U(S)]\notag\\
  &=-\sum_{S\in\mathcal S_k}p_U(S)\log p_U(S)
    =H(p_U),
  \label{eq:expected-proof-expanded}
\end{align}
which proves \eqref{eq:expected-gap}.

Let $N=\binom nk$ and define
$\mathcal B=\{S:p_U(S)<\delta/N\}$.  Then
\[
  \Pr(S\in\mathcal B)
  =\sum_{S\in\mathcal B}p_U(S)
  <|\mathcal B|\frac{\delta}{N}\le\delta.
\]
Outside $\mathcal B$, $\log p_U(S)\ge-\log(N/\delta)$, and
\eqref{eq:pointwise-log} gives \eqref{eq:high-prob}.  Monotone
post-processing can only decrease the certified gap.
\end{proof}

The high-probability argument uses only a counting fact.  There are at most
$N$ bad sets and each has probability less than $\delta/N$, so their total
probability is less than $\delta$.  It does not assume independence among
items within a DPP sample.

\begin{corollary}[Best of several independent samples]
\label{cor:multiple-samples}
Let $S_1,\ldots,S_r$ be independent draws from $p_U$, and return
$S_\star\in\arg\max_j\log\det(L_{S_j})$.  For any $\delta\in(0,1)$,
\begin{equation}
  \Pr\!\left[
    \log\det(L_{S_\star})
    \ge U_k-\log\!\left(\frac{\binom nk}{\delta^{1/r}}\right)
  \right]\ge1-\delta.
  \label{eq:best-of-r-bound}
\end{equation}
Moreover,
$\E[U_k-\log\det(L_{S_\star})]\le H(p_U)$.
\end{corollary}

\begin{proof}
Apply the one-sample bad-set argument with
$\alpha=\delta^{1/r}$.  Each sample misses the corresponding value threshold
with probability at most $\alpha$.  Independence makes the probability that
all $r$ samples miss it at most $\alpha^r=\delta$.  The returned best sample
therefore satisfies \eqref{eq:best-of-r-bound}.  Its gap is the minimum of the
$r$ individual gaps, hence is no larger in expectation than the gap of the
first draw, which is at most $H(p_U)$.
\end{proof}

\begin{remark}[What is and is not guaranteed]
The theorem guarantees a value relative to the spectral upper bound, not that
one random sample beats greedy on every instance.  When $H(p_U)$ is large,
the guarantee is correspondingly weak.  Multiple independent samples and
local improvement are practical ways to search the high-volume portion of
the distribution, while $\gap(S)$ reports the quality that was actually
certified.
\end{remark}

\subsection{Why spectral decay alone is insufficient}\label{sec:counterexample}

Consider $k=1$ and
\[
  L=\lambda uu\T+\epsilon I,
  \qquad u=\frac1{\sqrt n}\mathbf1,
  \qquad \lambda\gg n\epsilon>0.
\]
The leading eigenvalue is $\lambda+\epsilon$ and all other eigenvalues are
$\epsilon$, so the spectrum has an arbitrarily sharp drop.  Nevertheless,
\[
  U_1=\log(\lambda+\epsilon),
  \qquad
  \OPT_1=\log(\lambda/n+\epsilon),
\]
and the gap approaches $\log n$.  Here $p_U(\{i\})=1/n$ for every item, so
$H_\infty(p_U)=H(p_U)=\log n$.  The entropy description is exact, whereas
spectral decay alone predicts no rounding difficulty.

At the other extreme, if the leading eigenspace is spanned by $k$ coordinate
vectors, $p_U$ is a point mass, both entropies are zero, and the relaxation is
tight.  These examples justify treating eigenspace roundability as a first-class
quantity.

\section{The Certified Rounding Pipeline}\label{sec:algorithm}

\OurMethod\ uses the theory above as an algorithm rather than as an existence
argument.  It returns both a subset and a certificate.

\begin{algorithm}[H]
\caption{\OurMethod: spectral certificate with projection-DPP rounding}
\label{alg:certdpp}
\begin{algorithmic}[1]
\Require PSD kernel operator $x\mapsto Lx$; target size $k$; number of samples
         $r\ge1$; optional monotone refinement budget $b\ge0$;
         optional verified eigenvalue upper bounds
         $\bar\lambda_i\ge\lambda_i(L)$.
\State Compute a leading basis $U=[u_1,\dots,u_k]$ with a block eigensolver
       and fix its tie-breaking rule if $k<n$ and
       $\lambda_k=\lambda_{k+1}$.
\State If verified bounds are available, set
       $\bar U_k\leftarrow\sum_{i=1}^k\log\bar\lambda_i$;
       otherwise set $\widehat U_k$ to the sum of the computed Ritz-value logs.
\For{$j=1,\dots,r$}
  \State Draw $S_j\sim\mathrm{DPP}(UU\T)$ using an exact-law
         projection-DPP sampler.
  \State Evaluate $f_j\leftarrow\log\det(L_{S_j})$.
\EndFor
\State $S\leftarrow S_{j^\star}$ where $j^\star\in\arg\max_j f_j$.
\State Optionally apply at most $b$ determinant-increasing swaps to $S$.
\State With verified bounds, set
       $\gap_{\mathrm{safe}}(S)\leftarrow\bar U_k-\log\det(L_S)$;
       otherwise set
       $\widehat\gap(S)\leftarrow\widehat U_k-\log\det(L_S)$.
\State \Return $S$, $\log\det(L_S)$, and either rigorous certificate
       $\gap_{\mathrm{safe}}(S)$ or explicitly labeled numerical estimate
       $\widehat\gap(S)$.
\end{algorithmic}
\end{algorithm}

The sampling guarantees in \cref{thm:rounding} apply to the exact leading
subspace.  If the computed basis is approximate, its projection DPP is still
sampled exactly conditional on that basis, but transferring the theorem to the
true leading subspace requires a separate subspace/distribution perturbation
bound.  The a posteriori certificate remains rigorous whenever the scalar
ceiling uses verified upper bounds, regardless of how the subset was produced.

\subsection{Exact sequential projection-DPP sampling}

For completeness, \cref{alg:projection-sampler} gives the sequential sampler
used in the experiments.  At an iteration where $V$ has $t$ orthonormal
columns, the residual leverage masses satisfy
\begin{equation}
  \sum_{i=1}^n\|e_i\T V\|_2^2
  =\tr(VV\T)=\tr(V\T V)=t.
  \label{eq:residual-mass-sum}
\end{equation}
Therefore $\ell_i/t$ is a valid categorical distribution.  The basis update
conditions the projection DPP on the selected item and reduces its rank from
$t$ to $t-1$.

\begin{algorithm}[H]
\caption{Sequential exact sampler for a rank-$k$ projection DPP}
\label{alg:projection-sampler}
\begin{algorithmic}[1]
\Require $U\in\R^{n\times k}$ with $U\T U=I_k$.
\State $V\leftarrow U$; $S\leftarrow\varnothing$.
\For{$t=k,k-1,\ldots,1$}
  \State $\ell_i\leftarrow\|e_i\T V\|_2^2$ for $i=1,\ldots,n$.
  \State Draw item $i$ with probability $\ell_i/t$; set
         $S\leftarrow S\cup\{i\}$.
  \If{$t>1$}
    \State Draw pivot column $j\in[t]$ with probability
           $V_{ij}^2/\ell_i$.
    \State Let $V_{-j}$ denote $V$ without column $j$ and set
           $W\leftarrow V_{-j}-V_{:j}(V_{i,-j}/V_{ij})$.
    \State Replace $V$ by an orthonormal basis for $\operatorname{col}(W)$
           using a thin QR factorization.
  \EndIf
\EndFor
\State \Return the sorted set $S$.
\end{algorithmic}
\end{algorithm}

The pivot in line 6 has $V_{ij}\ne0$ with probability one under its sampling
rule, so the division is defined.  Moreover, row $i$ of $W$ is exactly zero:
\begin{equation}
  W_{i,:}=V_{i,-j}-V_{ij}\frac{V_{i,-j}}{V_{ij}}=0.
  \label{eq:selected-row-zero}
\end{equation}
Consequently the selected item cannot be drawn again.  Orthogonalizing $W$
does not change its column space, so it does not change the conditional
projection kernel.  Repeating until the residual rank is zero returns exactly
$k$ distinct indices.  The standard projection-DPP sampling identity shows
that the probability of the final unordered set is $\det(U_S)^2$.

In the worked example of \cref{ex:three-by-two}, every initial row has
$\ell_i=2/3$.  Since $t=2$, the first item is selected with probability
$(2/3)/2=1/3$.  Conditional on selecting item 1 first, symmetry gives
probability $1/2$ to item 2 and $1/2$ to item 3 at the second stage.  The two
possible orders for each unordered pair add to total pair probability $1/3$,
matching the determinant calculation in the table.

An implementation using structured rank-one orthogonalization updates costs
$O(nk^2)$ for one standard sample.  A literal implementation that recomputes
a full thin QR at every stage can cost up to $O(nk^3)$ and is best regarded as
a transparent reference implementation.  Faster rejection-based methods have
expected cost $O(nk+k^3\log k)$, with additional samples from the same
projection DPP costing expected $O(k^3\log k)$ after preprocessing
\citep{barthelme2023faster}.  Other exact or approximate samplers can be used
within the pipeline \citep{gautier2017zonotope,derezinski2019fast}; however,
the exact distributional guarantees in \cref{thm:rounding} apply literally
only when the output law is $p_U$.

\paragraph{Eigenspace computation.}
Block Lanczos, LOBPCG, or randomized subspace iteration requires only products
with $L$.  If one block matvec with $k$ vectors costs $T_L(k)$, $t$ iterations
cost approximately
\begin{equation}
  O\!\left(t\,[T_L(k)+nk^2]\right),
  \label{eq:eigenspace-cost}
\end{equation}
where the $nk^2$ term accounts for block orthogonalization.  For
$L=\Phi\Phi\T+\epsilon I$ with $\Phi\in\R^{n\times d}$, the product
$LX=\Phi(\Phi\T X)+\epsilon X$ costs $O(ndk)$.  The $n\times n$ kernel need
not be formed.  In a row-sharded implementation, forming $\Phi\T X$ requires a
$d\times k$ reduction; this communication cost should not be confused with
the smaller $k\times k$ Rayleigh quotient.

\subsection{End-to-end cost and data requirements}
\Cref{tab:complexity} separates costs that are sometimes conflated.  The
rounding stage is linear in $n$ at fixed $k$, but the eigensolve depends on
the cost and convergence rate of the kernel operator.  Evaluating the final
subset requires access only to the $k\times k$ principal submatrix.

\begin{table}[t]
\centering
\small
\begin{tabular}{p{.25\linewidth}p{.28\linewidth}p{.37\linewidth}}
\toprule
Stage & Representative cost & Required object \\
\midrule
Top-$k$ eigenspace & $O(t[T_L(k)+nk^2])$ & Block products $LX$ \\
Optimized sequential sample & $O(nk^2)$ & $U\in\R^{n\times k}$ \\
Reference re-QR sample & up to $O(nk^3)$ & $U\in\R^{n\times k}$ \\
Candidate log determinant & $O(k^3)$ after entries are read & $L_S$ \\
Naive full one-swap sweep & $O(k(n-k)k^3)$ & Candidate principal minors \\
One-swap sweep with updates & about $O(k(n-k)k^2)$ & Cholesky factors and kernel entries \\
\bottomrule
\end{tabular}
\caption{\textbf{Representative computational costs.}  Here $t$ is the
number of block-eigensolver iterations and $T_L(k)$ is the cost of multiplying
$L$ by an $n\times k$ block.  Exact constants and update costs depend on data
access and numerical implementation.}
\label{tab:complexity}
\end{table}

\paragraph{Local refinement and fair accounting.}
The theorem requires no local search.  A swap stage is optional and can only
improve the determinant.  Its cost must nevertheless be reported separately.
A naive sweep over a candidate pool of size $m$ costs $O(mk^3)$ determinant
work per replaced position; Cholesky rank-one updates reduce the per-candidate
cost to $O(k^2)$.  In experiments comparing initializers, we give every method
the same refinement budget so that improvements are not incorrectly attributed
to the continuous relaxation.

\subsection{Implementation checklist}
A careful implementation can be organized as the following verifiable steps.
\begin{enumerate}[leftmargin=*,itemsep=4pt]
  \item \textbf{Validate the target.} Check $1\le k\le n$ and confirm that
        at least $k$ eigenvalues are positive at the working tolerance.
  \item \textbf{Compute and orthogonalize the leading basis.} Form $U$ and
        check $\|U\T U-I_k\|$; reorthogonalize when necessary.
  \item \textbf{Construct the spectral ceiling.} Use verified upper bounds
        for the leading eigenvalues if a rigorous certificate is required.
  \item \textbf{Round.} Run \cref{alg:projection-sampler} one or more times,
        recording the random seed and whether the sampler is exact or
        approximate.
  \item \textbf{Evaluate stably.} Build $L_S$, compute its Cholesky factor,
        and obtain the log determinant from the diagonal rather than forming
        a determinant directly.
  \item \textbf{Refine monotonically.} Accept a swap only after verifying that
        it increases the full-kernel log determinant.
  \item \textbf{Report both numbers.} Return the achieved value
        $\log\det(L_S)$ and the gap $U_k-\log\det(L_S)$; do not report the gap
        without identifying the upper bound used to compute it.
\end{enumerate}

\subsection{Numerically valid certificates}\label{sec:numerical-cert}

Theorems above use exact leading eigenvalues.  A floating-point Ritz value is
not automatically a rigorous upper bound.  A production implementation that
advertises a mathematical certificate should use verified eigenvalue upper
bounds from the eigensolver.  More generally, suppose a matrix-free
approximation $\widetilde L$ satisfies
$\|L-\widetilde L\|_2\le\varepsilon_L$.  Weyl's inequality gives
\[
  \lambda_i(L)\le\lambda_i(\widetilde L)+\varepsilon_L,
\]
and therefore the conservative certificate
\begin{equation}
  U_k^{\mathrm{safe}}
  :=\sum_{i=1}^k
       \log\bigl(\lambda_i(\widetilde L)+\varepsilon_L\bigr)
  \ge U_k.
  \label{eq:safe-certificate}
\end{equation}
If no verified residual or kernel-approximation bound is available, the reported
gap should be labeled a numerical estimate rather than a rigorous certificate.

The subset determinant should also be evaluated in log space.  If
$L_S=RR\T$ is a Cholesky factorization with positive diagonal, then
\begin{equation}
  \log\det(L_S)=2\sum_{j=1}^k\log R_{jj}.
  \label{eq:cholesky-logdet}
\end{equation}
This avoids overflow and underflow in $\det(L_S)$.  A Cholesky failure may
indicate a singular subset, an indefinite input, or roundoff.  Adding a
diagonal ``jitter'' changes the optimization problem and must be declared as
part of the kernel rather than silently used only during evaluation.

Approximation affects the two guarantees differently.  Any valid upper bound
$U_k^{\mathrm{safe}}\ge\OPT_k$ certifies any returned subset, even if that
subset came from an approximate eigenspace.  In contrast, the entropy and
sampling guarantees in \cref{cor:minentropy-gap,thm:rounding} concern the
exact leading subspace of $L$; using an approximate $U$ requires a separate
perturbation or distributional-error analysis if those guarantees are to be
claimed rigorously.

\section{Experiments}\label{sec:experiments}

The experiments are designed to test the claims proved above.  They do not
attempt to establish downstream model-quality gains.  All experiments use
Python 3.14.6, NumPy 2.5.1, SciPy 1.18.0, and Matplotlib 3.11.1 in double
precision, using Apple's Accelerate BLAS/LAPACK backend on macOS 26.4 and fixed
random seeds.  The complete script is
\path{certificate_experiments.py}, included with the source; it writes raw
small-instance and timing results to CSV and records the full NumPy build
configuration in \path{certificate_environment.txt}.  Exact package versions
are also listed in \path{certificate_requirements.txt}.  The script validates
the projection-DPP sampler against enumerated probabilities on a $7$-item instance.
All experimental spectral ceilings below are computed by
\texttt{numpy.linalg.eigh}; because those values are not outward-rounded, we
treat the resulting gaps as numerical estimates rather than rigorous
certificates.

\subsection{Controlled roundability geometry}\label{sec:geometry-exp}

We construct a rank-$k$ family where the eigenspace localization can be varied
without changing the nonzero eigenvalues.  Set $k=5$, divide $n=25$ coordinates
into five disjoint groups of five, and give column $j$ of $U$ support only on
group $j$.  Within each group, one row receives squared mass $c$ and the other
four receive $(1-c)/4$.  The columns are orthonormal for every
$c\in[0.2,1]$.  We use nonzero eigenvalues
$(9,6,4,2.5,1.5)$ and draw 500 samples per setting.

Because the column supports are disjoint, a nonsingular size-$k$ row subset
must select exactly one row from each group.  The projection-DPP distribution
therefore factorizes into $k$ categorical distributions with weights
\begin{equation}
  w(c)=\left(c,\frac{1-c}{4},\frac{1-c}{4},
                 \frac{1-c}{4},\frac{1-c}{4}\right).
  \label{eq:grouped-weights}
\end{equation}
For $c\ge1/5$, the most likely subset chooses the first row from every group,
so its probability is $c^k$ and
\begin{equation}
  H_\infty(p_U)=-k\log c.
  \label{eq:grouped-minentropy}
\end{equation}
Independence across the group choices also gives
\begin{equation}
  H(p_U)=k\left[-c\log c
    -(1-c)\log\left(\frac{1-c}{4}\right)\right].
  \label{eq:grouped-shannon}
\end{equation}
These closed forms are the reference curves used in the first two panels.

\begin{figure}[t]
\centering
\includegraphics[width=\linewidth]{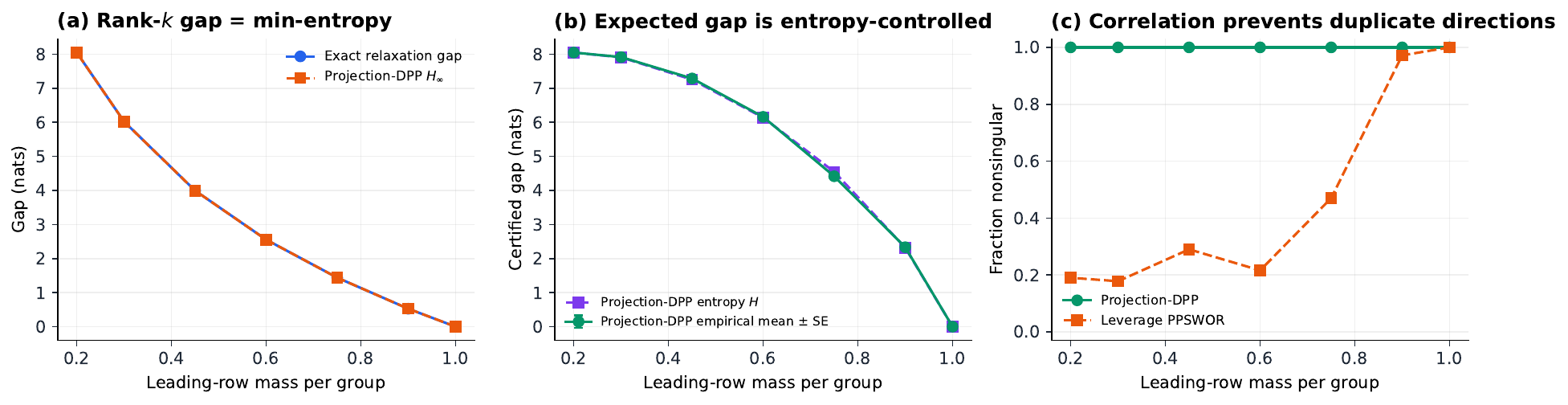}
\caption{\textbf{Eigenspace entropy, not eigenvalue decay, controls rounding.}
\textbf{(a)} For rank-$k$ kernels, the exact relaxation gap coincides with the
projection-DPP min-entropy, as predicted by
\cref{cor:minentropy-gap}.  \textbf{(b)} The empirical mean of the directly
evaluated gap $U_k-\log\det(L_S)$, with standard-error bars, coincides with the
Shannon entropy.  \textbf{(c)} Projection-DPP correlation always selects one
independent direction from each group; leverage-score
probability-proportional-to-size sampling without replacement (PPSWOR) often
selects several rows from one group and misses another, producing a singular
rank-$k$ submatrix.  PPSWOR is a weighted fixed-size comparator, not an
independent design with projection-DPP marginals.}
\label{fig:geometry}
\end{figure}

\Cref{fig:geometry}(a) verifies the exact rank-$k$ identity
$U_k-\OPT_k=H_\infty(p_U)$.  At $c=1$, the eigenspace is coordinate-aligned
and the gap is zero.  At $c=0.2$, its volume is evenly distributed among
$5^5$ nonsingular coordinate bases and the gap is
$\log(5^5)\approx8.05$ nats, even though the eigenvalues are unchanged.
Panel (b) verifies the expected-gap identity in the rank-$k$ case.  Panel (c)
also shows why replacing a projection DPP with leverage-weighted PPSWOR is not
innocuous: weighting individual items by leverage does not retain the joint
volume constraint.  We make no claim that this comparator preserves the
projection-DPP first-order inclusion probabilities.

\subsection{Exact small-instance comparisons}\label{sec:exact-exp}

We next generate 30 dense $18\times18$ PSD kernels with $k=5$, eigenvalues
$(12,7,4,2,1,0.9,\ldots,0.03)$, and Haar-random eigenvectors.  Since
$\binom{18}{5}=8568$, exhaustive enumeration gives the exact discrete optimum.
We compare standard greedy, column-pivoted QR on $U\T$, the best of 32
leverage PPSWOR samples, and the best of 32 projection-DPP samples.  Each
initializer receives the same single full one-swap refinement sweep.  We fix
the 30 matrix seeds to $0,\ldots,29$ before evaluation.  Table entries aggregate
gaps to exact OPT; lower is better.

More precisely, for each seed we draw a standard Gaussian matrix, compute its
QR factorization $Q$, and set $L=Q\diag(\lambda)Q\T$.  For an initializer
$S^{(0)}$ and its refined output $S^{(1)}$, we compute
\begin{equation}
  \OPT_k-\log\det(L_{S^{(0)}})
  \quad/\quad
  \OPT_k-\log\det(L_{S^{(1)}}).
  \label{eq:experiment-gap-pair}
\end{equation}
These gaps can be computed only because exhaustive enumeration is feasible at
this small size.

\begin{table}[H]
\centering
\small
\setlength{\tabcolsep}{7pt}
\begin{tabular}{lcccc}
\toprule
Gap to exact OPT & Greedy & QRCP & PPSWOR-32 & \OurMethod-32 \\
\midrule
Before refinement & $.049\pm.069$ & $.144\pm.160$ &
  $.400\pm.208$ & $.233\pm.214$ \\
After refinement & $.012\pm.026$ & $.031\pm.052$ &
  $.049\pm.082$ & $.028\pm.051$ \\
\bottomrule
\end{tabular}
\caption{\textbf{Exact small-instance comparison over 30 prespecified seeds.}
Entries are mean $\pm$ sample standard deviation in nats relative to exhaustive
MAP.  All methods receive the same one-sweep refinement budget.}
\label{tab:exact}
\end{table}

Three conclusions matter.  First, greedy is strongest on average in this dense
small-instance family; \OurMethod\ is not claimed to dominate every
initializer.  Second, equal swap budgets substantially reduce dependence on
initialization.  Third, the numerical spectral ceiling is conservative here:
its mean excess over OPT is $3.651\pm0.407$ nats, while the refined
\OurMethod\ output is only $0.028\pm0.051$ nats below OPT.  Consequently the
mean numerical gap estimate for that output is $3.679\pm0.420$ nats.  These
values are numerical estimates from \texttt{numpy.linalg.eigh}, not
outward-rounded certificates; the experiment exposes both looseness and
floating-point status rather than concealing either.

\subsection{Scaling of exact-law projection-DPP rounding}\label{sec:scaling-exp}

We isolate the rounding stage by generating an orthonormal
$U\in\R^{n\times k}$ and timing one sample from the standard sequential
exact-law projection-DPP sampler.  The experiment uses an Apple M4 Max CPU,
performs two untimed warm-up samples for each $(n,k)$ pair, and reports the
median and interquartile range of ten timed runs.  It excludes eigenspace
construction and optional local swaps.  The experiment script implements the
transparent re-QR update in \cref{alg:projection-sampler}; the purpose is to
measure the dependence on $n$ at fixed $k$, not to claim an optimized
dependence on $k$.

\begin{figure}[t]
\centering
\includegraphics[width=.64\linewidth]{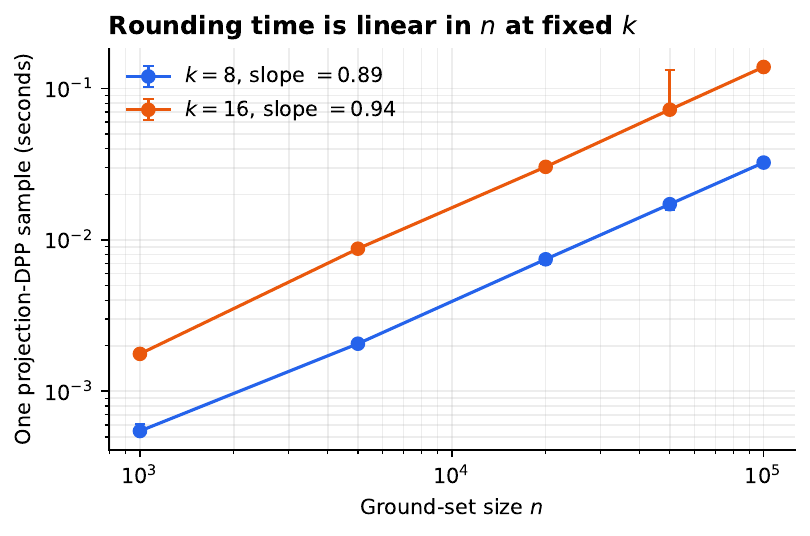}
\caption{\textbf{Scaling of the exact-law rounding stage.} Points and error
bars are medians and interquartile ranges over ten runs after two warm-ups.
The fitted log--log slopes are $0.89$ for $k=8$ and $0.94$ for $k=16$.  At
$n=100{,}000$, median times are $0.032$ seconds and $0.139$ seconds,
respectively.  The plot excludes leading-eigenspace computation.}
\label{fig:scaling}
\end{figure}

The fitted slopes in \cref{fig:scaling} are close to the predicted linear
dependence on $n$ over the tested range and implementation.  The eigenspace
computation will dominate when kernel matvecs are expensive; repeated sampling
reuses the same computed basis $U$.

\section{Discussion and Limitations}\label{sec:limitations}

\paragraph{The certificate may be loose.}
The worst-case width $\log\binom nk$ is large when $k\ll n$.  The min-entropy
bound is more informative, but computing $\max_Sp_U(S)$ is itself a
maximum-volume submatrix problem.  In practice, $\gap(S)$ is the immediately
computable diagnostic; entropy and min-entropy can be estimated or bounded to
explain why the diagnostic is tight or loose.

\paragraph{Sampling is not MAP optimization.}
A projection DPP samples in proportion to coordinate volume; it does not
directly maximize that volume.  Repetition increases the chance of seeing a
high-probability subset, and local swaps improve the realized determinant, but
neither step changes the NP-hardness of exact MAP.  The contribution is a
constructive certificate-relative guarantee, not a polynomial-time exact
solver.

\paragraph{The leading subspace can miss tail-assisted subsets.}
For full-rank $L$, the spectral tail $R_k$ may make a subset attractive even
when $p_U(S)$ is modest.  The pointwise theorem remains valid because the tail
only increases the determinant, but an algorithm based solely on $U$ may not
rank such subsets well.  Greedy or swap refinement on the full kernel is a
useful complement.

\paragraph{Approximate kernels need error accounting.}
Nystr\"om and random-feature kernels change both the eigenvalues and the
discrete objective.  A certificate for the approximation is not automatically
a certificate for the original kernel.  A spectral-norm error bound such as
\eqref{eq:safe-certificate}, or a clearly labeled approximate certificate, is
required.

\paragraph{Empirical scope.}
The present experiments verify the mathematical mechanisms, exact small-case
behavior, and rounding-stage scaling.  They do not establish downstream gains
for language-model curation, active learning, recommendation, or experimental
design.  Such application studies require domain-specific kernels, quality
weights, and comparisons with specialized selection baselines.

\section{Conclusion}\label{sec:conclusion}

The Stiefel relaxation of determinant MAP selection is exactly solvable, and
its optimum is the classical spectral upper bound $U_k$.  That fact removes
the need for a new continuous solver but does not remove the rounding problem.
The leading eigenspace carries a projection-DPP distribution whose probabilities
measure coordinate volume.  Relating those probabilities to the original
principal minors yields a pointwise determinant floor, an exact rank-$k$
min-entropy characterization of the relaxation gap, and expected and
high-probability guarantees for projection-DPP rounding.

The resulting message is deliberately modest and useful.  Spectral information
provides an efficiently computable ceiling; projection-DPP sampling provides a
constructive way to turn its eigenspace into a subset; local search can improve
the subset without invalidating the theory; and the final spectral gap reports
how much optimality has actually been certified.  When the gap is large, the
method says so.  Understanding and reporting that roundability gap is the main
purpose of the spectral relaxation.

\appendix
\section{Additional Proof Details}\label{app:proofs}

\subsection{Cauchy--Binet normalization}

\begin{lemma}[The projection volumes sum to one]
\label{lem:cauchy-binet-normalization}
If $U\in\R^{n\times k}$ and $U\T U=I_k$, then
\begin{equation}
  \sum_{S\subseteq[n],\ |S|=k}\det(U_S)^2=1.
  \label{eq:cauchy-binet-appendix}
\end{equation}
\end{lemma}

\begin{proof}
The Cauchy--Binet formula states that for
$A\in\R^{k\times n}$ and $B\in\R^{n\times k}$,
\begin{equation}
  \det(AB)=\sum_{|S|=k}\det(A_{:,S})\det(B_{S,:}).
  \label{eq:cauchy-binet-general}
\end{equation}
Set $A=U\T$ and $B=U$.  Then $A_{:,S}=U_S\T$ and
$B_{S,:}=U_S$, so each summand is
$\det(U_S\T)\det(U_S)=\det(U_S)^2$.  The left-hand side is
$\det(U\T U)=\det(I_k)=1$.
\end{proof}

\subsection{The entropy chain}

\begin{lemma}[Min-entropy, Shannon entropy, and support size]
\label{lem:entropy-chain}
For any probability distribution $p$ supported on at most $N$ outcomes,
\begin{equation}
  0\le -\log\max_x p(x)\le-\sum_x p(x)\log p(x)\le\log N.
  \label{eq:entropy-chain-appendix}
\end{equation}
\end{lemma}

\begin{proof}
Let $p_{\max}=\max_xp(x)$.  Since $p(x)\le p_{\max}\le1$ on the support,
$-\log p(x)\ge-\log p_{\max}\ge0$.  Taking the expectation under $p$ gives
$H(p)\ge H_\infty(p)\ge0$.  For the upper bound, let $u$ be uniform on a set
of $N$ outcomes, extending $p$ by zeros if needed.  Nonnegativity of relative
entropy gives
\begin{align}
  0\le D_{\mathrm{KL}}(p\|u)
  &=\sum_xp(x)\log\frac{p(x)}{1/N}
    =-H(p)+\log N,
\end{align}
so $H(p)\le\log N$.
\end{proof}

\subsection{Why the sequential sampler is exact}

\begin{proposition}[Conditional projection update]
\label{prop:conditional-projection}
Let $K=VV\T$ be an orthogonal projector of rank $t$, and suppose
$K_{ii}>0$.  Conditional on including item $i$, the remaining items form a
rank-$(t-1)$ projection DPP with projector
\begin{equation}
  K^{(i)}=K-\frac{K_{:i}K_{i:}}{K_{ii}}.
  \label{eq:conditional-projector}
\end{equation}
The matrix $W$ in line 7 of \cref{alg:projection-sampler} spans the range of
$K^{(i)}$.
\end{proposition}

\begin{proof}
The range of $K$ is $\operatorname{col}(V)$.  The vector
$K_{:i}=VV\T e_i$ is the projection of coordinate vector $e_i$ onto this
range.  Subtracting the rank-one projector
$K_{:i}K_{i:}/K_{ii}$ removes exactly that direction, leaving the
$(t-1)$-dimensional subspace
\begin{equation}
  \{Va:a\in\R^t,\ e_i\T Va=0\}.
  \label{eq:conditional-subspace}
\end{equation}
The columns of $W$ are linear combinations of columns of $V$, and
\eqref{eq:selected-row-zero} shows that they lie in this subspace.  The
chosen pivot is nonzero, and the $t-1$ columns are independent, so they span
the entire subspace.  Their orthonormalized basis therefore has projector
$K^{(i)}$.

For a projection DPP of rank $t$, the inclusion probability of $i$ is
$K_{ii}$.  Dividing by the fixed cardinality $t$ selects a uniformly random
member from the eventual set, which is why line 4 uses $K_{ii}/t$.  Applying
the conditional update recursively samples the remaining projection DPP.
Induction on $t$ then gives final unordered-set probability
$\det(K_S)=\det(V_S)^2$.
\end{proof}

\subsection{Determinant monotonicity}

\begin{lemma}[Determinant monotonicity used in \cref{thm:pointwise}]
If $A\succ0$ and $B\succeq A$, then $\det(B)\ge\det(A)$.
\end{lemma}

\begin{proof}
Write
$B=A^{1/2}(I+A^{-1/2}(B-A)A^{-1/2})A^{1/2}$.  The middle factor has all
eigenvalues at least one, so its determinant is at least one.  Therefore
$\det(B)\ge\det(A)$.
\end{proof}

\begin{remark}[Recovery of the classical counting floor]
Cauchy--Binet implies that at least one $S$ satisfies
$p_U(S)\ge1/\binom nk$.  Applying \cref{thm:pointwise} to that subset gives
\[
  \OPT_k\ge U_k-\log\binom nk.
\]
The projection-DPP proof is stronger than this averaging argument because it
provides a sampler whose expected and high-probability values obey parallel
bounds.
\end{remark}

\section{Reproducibility Details}\label{app:reproducibility}

The projection-DPP sampler in \texttt{certificate\_experiments.py} was checked
on a random $7\times3$ orthonormal matrix.  We enumerated all
$\binom73=35$ probabilities $\det(U_S)^2$, drew $30{,}000$ samples, and
observed maximum absolute empirical probability error $0.0025$.  The geometry
experiment uses seed 20260808; the exact-instance matrices use the 30
prespecified seeds $0,\ldots,29$; and the scaling experiment uses seed 777.
The geometry means use 500 samples per setting and report standard errors.  The
scaling study uses two untimed warm-ups and ten timed repetitions per setting,
reporting medians and interquartile ranges.  Figures and CSV files are
regenerated by running
\begin{quote}
\texttt{python3 certificate\_experiments.py}.
\end{quote}


\end{document}